\documentclass[11pt]{article}
\usepackage[margin=1in]{geometry} 

\usepackage{graphicx}
\usepackage{multirow} 
\usepackage{array,booktabs,arydshln,xcolor}

\usepackage{listings}
\usepackage{enumerate}
\usepackage{enumitem}
\usepackage{float}
\usepackage{tikz}
\usetikzlibrary{arrows}
\usepackage{natbib}

\usepackage{algorithmic}

\usepackage{caption}
\usepackage{mathtools}
\usepackage{hyperref}
\usepackage{fancyhdr}
\usepackage{relsize}

\usepackage{amsfonts}
\usepackage{amsmath}
\usepackage{amssymb}
\usepackage{bbm}
\usepackage{bm}
\usepackage{adjustbox}
\usepackage{ifthen}
\usepackage[ruled,linesnumbered]{algorithm2e}
\usepackage{amsthm, amssymb} 

\definecolor{toc}{RGB}{13,55,174}	
\usepackage{hyperref}				
\hypersetup{
    colorlinks=true,
    citecolor=red,
    filecolor=black,
    linkcolor=toc,
    urlcolor=toc
}

\allowdisplaybreaks 

\usepackage{thmtools}
\usepackage{thm-restate}

\newtheorem{theorem}{Theorem}[section]

\newtheorem{lemma}[theorem]{Lemma}

\newtheorem*{maintheorem}{Main Theorem}
\newtheorem{assumption}[theorem]{Assumption}

\newcommand{\dist}{\mathcal{D}}
\newcommand{\boxes}{\mathcal{B}}
\newcommand{\reals}{\mathbb{R}}
\newcommand{\lp}{\left}
\newcommand{\rp}{\right}
\newcommand{\relu}{\textrm{\normalfont ReLU}}
\newcommand{\E}[2]{\mathbb{E}_{#1}\lp[#2 \rp]}

\newcommand{\hs}{\hat{\sigma}}
\newcommand{\bhs}{\hat{\bm{\sigma}}}
\newcommand{\bdist}{\bm{ \mathcal{D}}}

\newcommand{\weitz}{\text{\normalfont WEITZ}}

\newtheorem{definition}{Definition}

\newcommand{\norm}[2]{|| #1 ||_{#2}}

\newcommand{\pb}{\textsc{Pandora's Box}}
\newcommand{\cpb}{\textsc{Contextual Pandora's Box}}

\newcommand{\s}{\sigma}
\usepackage[capitalize]{cleveref}
\title{Contextual Pandora's Box}
\author{
Alexia Atsidakou\footnote{Authors listed in alphabetical order.} \\ UT-Austin \\ {\tt atsidakou@utexas.edu} \and
Constantine Caramanis\\ UT-Austin \\ {\tt constantine@utexas.edu} \and
Evangelia Gergatsouli \\ UW-Madison \\ {\tt gergatsouli@wisc.edu} \and
Orestis Papadigenopoulos\\ Columbia University \\ {\tt papadig@cs.utexas.edu} \and
Christos Tzamos \\ UW-Madison \\ {\tt tzamos@wisc.edu}
}

\date{}
\begin{document}

\maketitle

\begin{abstract}
\pb{} is a fundamental stochastic optimization problem, where the decision-maker must find a good alternative, while minimizing the search cost of exploring the value of each alternative. In the original formulation, it is assumed that accurate distributions are given for the values of all the alternatives, while recent work studies the online variant of \pb{} where the distributions are originally unknown. In this work, we study \pb{} in the online setting, while incorporating context. At each round, we are presented with a number of alternatives each having a context, an exploration cost and an unknown value drawn from an unknown distribution that may change at every round. Our main result is a no-regret algorithm that performs comparably well against the optimal algorithm which knows all prior distributions exactly. Our algorithm works even in the bandit setting where the algorithm never learns the values of the alternatives that were not explored. The key technique that enables our result is a novel  modification of the realizability condition in contextual bandits that connects a context to a sufficient statistic of each alternative's distribution (its \emph{reservation value}) rather than its mean.
\end{abstract}

\section{Introduction}

\pb{} is a fundamental stochastic optimization framework, which models the trade-off between exploring a set of alternatives and exploiting the already collected information, in environments where data acquisition comes at a cost. In the original formulation of the problem -- introduced by Weitzman in \citep{Weit1979} -- a decision-maker is presented with a set of alternatives (called ``boxes'') each containing an unknown value drawn independently from some known box-specific distribution. In addition, each box is associated with a known cost, namely, the price that needs to be paid in order to observe its realized value. At each step, the decision-maker can either open a box of her choice, paying the associated cost and observing its value, or stop and collect the minimum value contained in the already opened boxes. The objective is to minimize the sum of the smallest observed value plus the total cost incurred by opening the boxes.

The model captures a variety of different settings where the decision-maker needs to balance the value of the selected alternative and the effort devoted to find it. We include some examples below.
\begin{itemize}
    \item Consider an online shopping environment where a search engine needs to present users with results on a product they want to buy. 
    Visiting all potential e-shops that sell this product to find the cheapest option would be prohibitive in terms of time needed to present the search results to the user. The search engine needs to explore the different options only up to the extent that would make the marginal improvement in the best price found worthwhile.
    \item Consider a path planning service provider like Google Maps. Upon request, the provider must search its database for a good path to recommend to the user, but the higher the time spent searching the higher is the server cost. The provider must trade off computation cost with the quality of the result.
\end{itemize}

Rather surprisingly, despite the richness of the setting, Weitzman shows that the optimal policy for any instance of \pb{} admits a particularly simple characterization: for each box, one can compute a {\em reservation value} as a function of its cost and value distribution. Then, the boxes are inspected in increasing order of these values, until a simple termination criterion is fulfilled. This characterization is revealing: obtaining an optimal algorithm for \pb{} does not require complete knowledge of the distributions or the costs, yet only access to a {\em single statistic} for each box.

This raises the important question of how easy it is to learn a near-optimal search strategy, especially in environments where the distributions may not remain fixed across time but can change according to the characteristics of the instance at hand. In the case of online shopping, depending on the type of product we are searching, the e-shops have different product-specific distributions on the prices. 
We would be interested in a searching strategy  that is not tied to a specific product, but is able to minimize the expected cost for any product we may be interested in. Similarly, in the case of path recommendations, the optimal search strategy may depend on the time of day or the day of the year.

Motivated by the above question, we extend the \pb{} model to a contextual online learning setting, where the learner faces a new instance of the problem at each round. At the beginning of each round, the learner observes a context and must choose a search strategy for opening boxes depending on the observation. While the context and the associated opening costs of the boxes of each round are observed, the learner has no access to the value distributions, which may be arbitrary in each round. 

\paragraph{Realizability.}
In the above description, the context may be irrelevant to the realized values. 
Such an adversarial setting is impossible to solve as it is related to the problem of learning thresholds online. In fact, even in the offline version of the problem, the task would still be computationally hard as it corresponds to agnostic learning (see \cref{apn:hardness} for more details). This naturally raises the following question:
\emph{What are the least possible strong assumptions, under which the problem becomes tractable?}

One of the main contributions of this paper is identifying a minimal realizability assumption under which the problem remains tractable. This assumption is parallel to the realizability assumption in contextual bandits (see chapter 19 of \citep{lattimore-Bandit}).
We describe below how the difficulty of problem increases as we move from the strongest (1) to the weakest (4) assumptions possible:
\begin{enumerate}
    \item \textbf{Contexts directly related to values}: in this case any learning algorithm is able to fit the contexts and predict exactly the realized values of the boxes. Such a setting is trivial yet unrealistic.
    
     \item \textbf{Contexts directly related to distributions}: this is the case where there exists a learnable mapping from the context to the distribution of values. This is a more realistic setting, but it is still relatively constrained; it requires being able to perfectly determine the distribution family, which would need to be parametric. 
     
    \item \textbf{Contexts related to sufficient statistic}: in the more general case, instead of the whole value distributions, the contexts give us information about only a sufficient statistic of the problem.
    This is one of the main contribution of this work; we show that the problem remains tractable when the contexts give us information on the \emph{reservation values} of the boxes. Observe that in this case, the value distributions on the boxes can be arbitrarily different at each round, as long as they ``implement'' the correct reservation value based on the context. This model naturally extends the standard realizability assumption made in bandit settings, according to which, the mean of the distributions is predictable from the context. In that case, the sufficient statistic needed in order to select good arms is, indeed, the mean reward of each arm \citep{lattimore-Bandit}.
    
    \item \textbf{No assumptions}: in this case, as explained before, the problem becomes intractable (see \cref{apn:hardness}).
\end{enumerate}

\subsection{Our Contribution}
We introduce a novel contextual online learning variant of the \pb{} problem, namely the \cpb{}, that captures the problem of learning near-optimal search strategies in a variety of settings. 

Our main technical result shows that even when the sequence of contexts and distributions is adversarial, we can find a search strategy with sublinear average regret (compared to an optimal one) as long as no-regret algorithms exist for a much simpler online regression problem with a linear-quadratic loss function.

\begin{maintheorem}[Informal]
Given an oracle that achieves  expected regret $r(T)$ after $T$ rounds for Linear-Quadratic Online Regression, there is an algorithm that obtains $O( \sqrt{ T r(T) } )$ regret for the \cpb{} problem.
\end{maintheorem}

The main technical challenge in obtaining the result is that the class of search strategies can be very rich.  Even restricting to greedy policies based on reservation values for each box, the cost of the policies is a non-convex function of the reservation values. 
We manage to overcome this issue 
by considering a ``proxy'' function for the expected cost of the search policy which bounds the difference from the optimal cost, based on a novel sensitivity analysis of the original Weitzman's algorithm.
The proxy function has a simple linear-quadratic form and thus optimizing it reduces our setting to an instance of {\em linear-quadratic online regression}. This allows us to leverage existing
methods for minimizing regret in online regression problems in a black box manner.

Using the above reduction, we design algorithms with sublinear regret guarantees for two different variants of our problem: the {\em full information}, where the decision-maker observes the  realized values of all boxes at the end of each round, and the {\em bandit} version, where only the realized values of the opened boxes can be observed. We achieve both results by constructing oracles based on the {\em Follow the Regularized Leader} family of algorithms.

Beyond the results shown in this paper, an important conceptual contribution of our work is extending the traditional bandit model in the context of stochastic optimization. Instead of trying to learn simple decision rules, in stochastic optimization we are interested in learning complex algorithms tailored to a distribution. Our model can be extended to a variety of such problems beyond the \pb{} setting: one concrete such example is the case of designing revenue optimal auctions for selling a single item to multiple buyers given distributional information about their values.
Modeling these settings through an online contextual bandit framework allows obtaining results without knowledge of the prior distributions which may change based on the context. Our novel realizability assumption allows one to focus on predicting only the sufficient statistics required for running a specific algorithm. In the case of designing revenue optimal auctions, contexts may refer to the attributes of the item for sale and a sufficient statistic for the bidder value distributions are Myerson's reserve prices~\cite{Myer1981}.

\subsection{Related Work}

We model our search problem using \pb {}, which was first introduced by Weitzman in the Economics literature \citep{Weit1979}. Since then, there has been a long line of research studying
\pb {} and its variants e.g. where boxes can be selected without inspection
\citep{Dova2018,BeyhKlein2019}, there is correlation between the boxes~\citep{ChawGergTengTzamZhan2020,ChawGergMcmaTzam2021}, the boxes have to be inspected in a specific order~\citep{BoodFuscLazoLeon2020} or boxes are inspected in an  online manner~\citep{EsfaHajiLuciMitz2019} or over $T$ rounds~\citep{GergTzam2022,GatmKessSingWang2022}. Some work is also done in the generalized setting where more information can be obtained for a
price~\citep{CharFagiGuruKleiRaghSaha2002,GuptKuma2001,ChenJavdKarbBagnSrinKrau2015, ChenHassKarbKrau2015}. Finally a long line of research considers more complex combinatorial constraints like budget constraints \citep{GoelGuhaMuna2006}, packing constraints \citep{GuptNaga2013}, matroid
constraints \citep{AdamSvirWard2016}, maximizing a submodular function \citep{GuptNagaSing2016, GuptNagaSing2017}, an approach via Markov chains \citep{GuptJianSing2019} and various packing and covering constraints for both minimization and maximization problems~\citep{Sing2018}.

A more recent line of work, that is very closely related to our setting, studies \pb{} and other stochastic optimization problems like min-sum set cover in online settings \citep{GergTzam2022, FotaLianPiliSkou2020}. Similar to our work, these papers do not require specific knowledge of distributions and provide efficient algorithms with low regret. However, these settings are not contextual, therefore they can only capture much simpler practical applications. Moreover, they only obtain multiplicative regret guarantees as, there, the problems are NP-hard to solve exactly.

Our problem is closely related to contextual multi-armed bandits, where the contexts provide additional information on the quality of the actions at each round. In particular, in the case of \textit{stochastic linear bandits} \citep{Abe99associativereinforcement}, the reward of each round is given by a (noisy) a linear function of the context drawn at each round. 
Optimistic algorithms proposed for this setting rely on maintaining a confidence ellipsoid for estimating the unknown vector \citep{Dani2008StochasticLO,Rusmevichientong2010LinearlyPB,Abbasi11,Valko2014SpectralBF}. 
On the other hand, in \textit{adversarial linear bandits}, a context vector is adversarially selected at each round. The loss is characterized by the inner product of the context and the selected action of the round. Common approaches for this setting include variants of the multiplicative-weights algorithm \citep{pmlr-v35-hazan14b,pmlr-v75-hoeven18a}, as well as, tools from online linear optimization \citep{Blair1985ProblemCA,Cesa-Bianchi_lugosi2006} such as follow-the-regularised-leader and
mirror descent (see \citep{pmlr-v40-Bubeck15b, Abernethy2008CompetingIT, ShalevShwartz2007APP, Bubeck2018SparsityVA} and references therein). 

We note that our model generalizes the contextual bandits setting, since any instance of contextual bandits can be reduced to \cpb{} for box costs selected to be large enough.
One work from the contextual bandits literature that is more closely related to ours is the recent work of~\citet{FostRakh2020}. Similarly to our work, they provide a generic reduction from contextual multi-armed bandits to online regression, by showing that any oracle for online regression can be used to obtain a contextual bandits algorithm.

Our work also fits in the recent direction of learning algorithms from data and algorithm configuration, initiated by~\citet{GuptRoug2016}, and continued in~\citep{BalcNagaViteWhit2016,BalcDickSandVite2018, BalcDickVite2018, KleiLeytLuci2017,WeisGyorSzep2018,AlabKalaLigeMuscTzamVite2019}. Similar work was done before in self-improving algorithms in \citep{AiloChazClarLiuMulzSesh2006, ClarMulzSesh2012}. A related branch of work initiated by~\citet{MunoVass2017} and~\citet{LykoVass2018} combines Machine Learning predictions to improve algorithm design  studying various online algorithm questions like ski rental \citep{PuroSvitKuma2018,GollPani2019,WangLiWang2020}, caching~\citep{LykoVass2018, Roha2020, Wei2020}, scheduling~\citep{LattLavaMoselVass2020,Mitz2020}, online primal-dual method~\citep{BamaMaggSven2020}, energy minimization~\citep{BamaMaggRohwSvens2020} and secretary problem/matching~\citep{AntoGoulKleeKole2020}. While most of these works assume that the predictions are given a priori, recent works in the area~\citep{DiakKontTzamVakiZar2021, LavaMoseRaviXu2021} also focus on the task of learning the predictions. Our work can also be seen as learning to predict for the reservation values of the boxes in \cpb.

\section{Problem Definition and Notation}\label{sec:prelims}
We begin by describing the original \pb{} formulation. Then, in \cref{sec:contextual_pandoras} we describe our online extension, which solves a contextual instance of \pb{} at each round.

\subsection{Original Pandora's Box Formulation}

In \pb {} we are given a set $\boxes$ of $n$ boxes each with cost $c_i$ and value $v_i\sim \dist_i$, where the distributions $\dist_i$ and the costs $c_i$ for each box are known and the distributions are independent. 
The goal is to adaptively choose a box of small cost while spending as little as possible in opening costs. When opening box $i$, the algorithm pays $c_i$ and observes the value $v_i\sim \dist_i$ instantiated inside the box. The formal definition follows.

\begin{definition}[\pb{} cost]\label{def:PB_cost}
Let $\mathcal{P}$ and $c_{i}$ be the set of boxes opened and the cost of the box selected, respectively. The cost of the algorithm is
\begin{equation}\label{eq:pandoras_cost}  \E{\mathcal{P},v_i\sim \dist_i \forall i\in \boxes}{\min_{i\in \mathcal{P}} v_i +
\sum_{i\in \mathcal{P}} c_i }
\end{equation}
where the expectation is taken over the distributions of the values $v_i$ and the (potentially random) choice of $\mathcal{P}$ by the algorithm.
\end{definition}

Observe that an adaptive algorithm for this problem has to decide on: (a) a (potentially adaptive) order according to which it examines the boxes and, (b) a stopping rule.
Weitzman's algorithm, first introduced in \cite{Weit1979}, and formally presented in Algorithm~\ref{algo:weitz_prelim}, gives a solution to \pb{}. 
The algorithm uses the order induced by the \textit{reservation values} to open the boxes.

\begin{algorithm}[H]
\caption{Weitzman's algorithm.}\label{algo:weitz_prelim}
	\KwIn{$n$ boxes, costs $c_i$ and reservation values $\s_i$ for $i\in \boxes$}
	$\pi \leftarrow $ sort boxes by increasing $\s_i$\\
	$v_{\min} \leftarrow \infty$\\
	\For{every box $\pi_i$ }{
			Pay $c_i$ and open box $\pi_i$ and observe $v_i$\\
			$v_{\min} \leftarrow \min(v_{\min}, v_i)$\\
			\If{$v_{\min} < {\s}_{\pi_{i+1}}$}{
				Stop and collect $v_{\min}$
			}
	}
\end{algorithm}

\mbox{}\\
We denote by $\weitz_{\bm \dist}(\bm{\s};\bm{c})$ the expected cost of running Weitzman's algorithm using reservation values $\bm{\s}$ on an instance with distribution $\bm {\dist}$ and costs $\bm{c}$.
Weitzman showed that this algorithm achieves the optimal expected cost of \cref{eq:pandoras_cost} for the following selection of reservation values:

\begin{theorem}[\citet{Weit1979}]\label{thm:weitzman}
Weitzman's algorithm is optimal for \pb{} when run with reservation values $\bm {\s}^*$ that satisfy $\E{v_i\sim \dist_i}{\relu(\s_i^*-v)} = c_i$ for every box $i \in \boxes$, where $\relu(x) = \max\{x,0\}$.
\end{theorem}

\subsection{Online Contextual Pandora's Box}\label{sec:contextual_pandoras}
We now describe an online contextual extension of \pb. In 
\cpb{} there is a set $\boxes$ of $n$ boxes with costs $\bm c = (c_1,\ldots,c_n)$\footnote{Every result in the paper holds even if the costs change for each $t\in[T]$ and can be adversarially selected.}.
At each round $t\in [T]$:

\begin{enumerate}
    \item An (unknown) product distribution $\bm \dist_t = (\dist_{t,1},\ldots,\dist_{t,n})$ is chosen and for every box $i$, a value $v_{t,i} \sim \dist_{t,i}$ is independently realized.

    \item A vector of contexts $\bm x_t = (\bm x_{t,1}, \ldots \bm x_{t,n})$ is given to the learner, where $\bm{x}_{t,i} \in \reals^d$ and $\|\bm{x}_{t,i}\|_2 \le 1$ for each box $i \in \boxes$.
    \item The learner decides on a (potentially adaptive) algorithm $\mathcal{A}$ based on past observations.
    \item The learner opens the boxes according to $\mathcal{A}$ and chooses the lowest value found.
    \item At the end of the round, the learner observes all the realized values of all boxes (\textit{full-information} model), or observes only the values of boxes opened in the round (\textit{bandit} model).
 \end{enumerate}

 \begin{assumption}[Realizability]
 There exist vectors $\bm{w}_1^*,\ldots,\bm{w}_n^* \in \reals^d$ and a function $h$, such that 
 for every time $t\in [T]$ and every box $i\in \boxes$,
 the optimal reservation value $\s^*_{t,i}$ for the distribution $\dist_{t,i}$ is equal to $h(\bm{w}_i^*, \bm x_{t,i})$, i.e. 
 \[
 \E{v_{t,i}\sim \dist_{t,i}}{\relu(h(\bm{w}_i^*, \bm x_{t,i}) - v_{t,i})} = c_i.
 \]
 \end{assumption} 
The goal is to achieve low expected regret over $T$ rounds compared to an optimal algorithm that has prior knowledge of the vectors $\bm{w}_i^*$ and, thus, can compute the exact reservation values of the boxes in each round and run Weitzman's optimal policy. The regret of an algorithm is defined as the difference between the cumulative \pb{} cost achieved by the algorithm compared to the cumulative cost achieved by running Weitzman's optimal policy at every round. That is:

\begin{definition}[Expected Regret]\label{def:expected_regret}
    The expected regret of an algorithm $\mathcal{A}$ that opens boxes
    $\mathcal{P}_t$ at round $t$ over a time horizon $T$ is
\begin{equation}\label{eq:regret}
\text{\normalfont Regret}(\mathcal{A}, T)=\nonumber
\E{}{\sum_{t=1}^T \lp( \min_{i\in \mathcal{P}_t} v_{t,i} +
\sum_{i\in \mathcal{P}_t} c_i  - \weitz_{\bm D_t}(h(\bm w^*, \bm x_t);\bm c) \rp)}.
\end{equation}
The expectation is taken over the randomness of the algorithm, the contexts $\bm{x}_t$, distributions $\bm{\dist}_t$ and realized values $\bm{v}_t\sim \bm{\dist}_t$, over all rounds $t\in[T]$.
\end{definition}

\paragraph{Remark.} If the learner uses Weitzman's algorithm at every time step $t$, with reservation values $\sigma_{t,i} = h(\bm w_{t,i},\bm x_{t,i})$ for some chosen parameter $\bm w_{t,i} \in \reals^d$ for every box $i \in \boxes$, the regret can be written as
\begin{align*}
\text{\normalfont Regret}(\mathcal{A}, T) =
\mathbb{E} \Big[ \sum_{t=1}^T \Big( \weitz_{\bm D_t}(h(\bm w_t, \bm x_t);\bm c)  - \weitz_{\bm D_t}(h(\bm w^*, \bm x_t);\bm c) \Big)\Big],
\end{align*}
where $\bm{w}_t = (\bm{w}_{t,1}, \ldots, \bm{w}_{t,n})$.

\section{Reduction to Online Regression}\label{sec:online_regression}

In this section we give a reduction from \cpb{} problem to an instance of online regression, while maintaining the regret guarantees given by online regression. We begin by formally defining the regression problem:

\begin{definition}[Linear-Quadratic Online Regression]\label{def:online_regression}
Online regression with loss $\ell$ is defined as follows: at every round $t$, the learner 
 first chooses a prediction $\bm w_t$, then an adversary chooses an input-output pair $(\bm x_t,y_t)$ and the learner incurs loss $\ell(  h(\bm w_t, \bm x_t) - y_t)$.

In the \textbf{costly feedback} setting, the learner may observe the input-output pair at the end of round $t$, if they choose to pay an information acquisition cost $a$. The \textbf{full information} setting, corresponds to the case where $a=0$, in which case the input-output pair is always visible.
The regret of the learner after $T$ rounds, when the learner has acquired information $k$ times, is equal to 
\[ \sum_{t=1}^T \ell(h(\bm w_t, \bm x_t) - y_t) - \min_{\bm w} \sum_{t=1}^T \ell(h(\bm w_t, \bm x_t) - y_t) + a k.\]
\textbf{Linear-Quadratic Online Regression} is the special case of online regression where the loss function $\ell(z)$ is chosen to be a linear-quadratic function of the form
\begin{equation}\label{eq:convex_proxy}
H_c(z) =  \frac{1}{2} \relu(z)^2 - c z
\end{equation}
for some parameter $c > 0$.
\end{definition}
The reduction presented in Algorithm~\ref{algo:main} shows how we can use an oracle for linear-quadratic online regression, to obtain an algorithm for \cpb{} problem. We show that our algorithm achieves $O(\sqrt{Tr(T)})$ regret when the given oracle has a regret guarantee of $r(T)$.

\begin{theorem}\label{thm:main}
Given an oracle that achieves  expected regret $r(T)$ for Linear-Quadratic Online Regression, Algorithm \ref{algo:main} achieves $2 n \sqrt{ T r(T) }$ regret for the \cpb{} problem.
In particular, if the regret $r(T)$ is sublinear in $T$,
Algorithm \ref{algo:main} achieves sublinear regret.
\end{theorem}

Our algorithm works by maintaining a regression oracle for each box, and using it at each round to obtain a prediction on $\bm{w}_{t,i}$. Specifically, in the prediction phase of the round the algorithm obtains a prediction and then uses the context $\bm{x}_{i,t}$ to calculate an estimated reservation value for each box. 
Then, based on the {estimated} reservation values, it uses Weitzman's algorithm \ref{algo:weitz_prelim} to decide which boxes to open. Finally, it accumulates the \pb{} cost acquired by Weitzman's play at this round and the cost of any extra boxes opened by the oracle in the update phase (in the bandit setting). 
The update step is used to model the full-information setting (where the value of each box is always revealed at the end of the round) vs the costly feedback setting (where the value inside each box is only revealed if we paid the opening cost).\\

\begin{algorithm}[H]
\caption{$\mathcal{CPB}$: Contextual Pandora's Box }\label{algo:main}
\KwIn{Input: }Oracle $\mathcal{O}$ for Linear-Quadratic Online Regression.\\
 For every box $i \in \boxes$, instantiate a copy $\mathcal{O}_i$ of the oracle with linear-quadratic loss $H_{c_i}$\\
\For{each round $t\in T$} {
\textcolor{blue}{\# Prediction Phase}\\
Call oracle $\mathcal{O}_i$ to get a prediction $\bm{w}_{t,i}$ for each box $i$ \\
Obtain context $\bm{x}_t^i$ for each box $i\in \boxes$\\
Run Weitzman's algo \ref{algo:weitz_prelim} with reservation values ${\s}_{t,i} = h(\bm{w}_{t,i}, \bm{x}_{t,i})$ for each box $i\in \boxes$\\
\mbox{}\\
\textcolor{blue}{\#Update Phase}\\
\For{every box $i\in \boxes$}{
    \If{the oracle $\mathcal{O}_i$ requests an input-output pair at this round} {
    Observe value $v_{t,i}$ 
    and give the input-output pair $(\bm{x}_{t,i}, v_{t,i})$\\
    }
}
}
\end{algorithm}
\mbox{}\\
In the rest of this section we outline the proof of \Cref{thm:main}. 
The proof is based on obtaining robustness guarantees for Weitzman's algorithm when it is run with estimates instead of the true reservation values. 
In this case, we show that the cost incurred by Weitzman's algorithm is proportional to the error of the \textit{approximate costs} of the boxes (\cref{def:apx_cost}). This analysis is found on \cref{sec:robust}. 
Then, in \cref{sec:concluding_full_info_th} we exploit the form of the Linear-Quadratic loss functions to connect the robustness result with the regret of the Linear-Quadratic Online Regression problem and conclude our main \Cref{thm:main} of this section. 

An empirical evaluation of \Cref{algo:main} can be found in \Cref{sec:experiments}.

\subsection{Weitzman's Robustness}\label{sec:robust}
We provide guarantees on Weitzman's algorithm~\ref{algo:weitz_prelim} performance when instead of the optimal reservation values $\bm{\s}^*$ of the boxes, the algorithm uses estimates ${\bm{\s}} \not = \bm{\s}^*$. We first define the following: 

\begin{definition}[Approximate Cost]\label{def:apx_cost}
Given a distribution $\dist$ such that $v\sim \dist$ and a value $\s$, the approximate cost with respect to $\s$ and $\dist$ is defined as:
\begin{align}
    c_{\dist}(\s) = \E{v \sim \dist}{ \relu(\s - v) }.
\end{align}
Moreover, given $n$ boxes with estimated reservation values $\bm{\s}$ and distributions $\bm{\dist}$ we denote the vector of approximate costs as $\bm c_{\bm \dist}(\bm{\s}) = (c_{\dist_1}(\s_1),\dots,c_{\dist_n}(\s_n))$.
\end{definition}

\paragraph{Remark.}Observe that, in the \pb{} setting, if box $i$ has value distribution $\dist_i$ opening cost $c_i$ and optimal reservation value $\s_i^*$, then, by definition, the quantity $c_{\dist_i}(\s_i^*)$ corresponds to the true cost, $c_i$, of the box. This also holds for the vector of approximate costs, i.e. $\bm c_{\bm \dist}({\bm\s}^*) = \bm c$.

We now state our robustness guarantee for Weitzman's algorithm. 
In particular, we show that the extra cost incurred due to the absence of initial knowledge of vectors $\bm{w}^*_i$ is proportional to the error of the approximate costs the boxes, as follows:

\begin{restatable}{proposition}{propCost}\label{thm:robust}
For a \pb{} instance with $n$ boxes with distributions $\bm \dist$, costs $\bm c$ and corresponding optimal reservation values $\bm \s^*$ so that $\bm{c} = \bm c_{\bm \dist}(\bm \s^*)$,
Weitzman's Algorithm \ref{algo:weitz_prelim}, run with reservation values 
$\bm \s$ incurs cost at most 
$$\weitz_{\bm \dist}(\bm{\s}; \bm c) \le \weitz_{\bm \dist}({\bm \s}^*; \bm c) + 
\|\bm c_{\bm \dist}(\bm{\s}) - \bm{c}\|_1.$$
\end{restatable}
Before showing \Cref{thm:robust}, we prove the following lemma, that connects the optimal \pb{} cost of an instance with optimal reservation values $\bm \s^*$ to the optimal cost of the instance with optimal reservation values $\bm\s$. 
\begin{restatable}{lemma}{lemOptHat}\label{lem:opt_hat}
Let $\weitz_{\bm\dist}(\bm\s^* ; \bm{c}) $ and $\weitz_{\bm\dist}(\bm\s ; \bm{c_\dist}(\bm{\s}))$ be the optimal \pb{} costs corresponding to instances with optimal reservation values $\bm\s ^*$ and $\bm{\s}$ respectively. Then 
\begin{align*} 
\weitz_{\bm\dist}({\bm\s}^* ; \bm{c}) \geq
\weitz_{\bm\dist}(\bm\s ; \bm{c_\dist}(\bm{\s})) - \sum_{i\in \boxes} \relu(c_{\dist_i}(\s_i)- c_i).\end{align*}
\end{restatable}
The proof of the above Lemma, together with that of \Cref{thm:robust}, are deferred to \cref{apn:robust}.

\subsection{Proof of \cref{thm:main}}\label{sec:concluding_full_info_th}

Moving on to show our main theorem, we connect the robustness \Cref{thm:robust} with the performance guarantee of the Linear-Quadratic Online Regression problem. The robustness guarantee of Weitzman's algorithm is expressed in terms of the error of the approximate costs of the boxes, while the regret of the Online Regression problem is measured in terms of the cumulative difference of the linear-quadratic loss functions $H_c(\cdot)$. Thus, we begin with the following lemma:

\begin{restatable}{lemma}{lemmaCtoF}\label{lemma:c_to_f}
For any distribution $\dist$ with $c_{\dist}(\s^*) = c$, it holds that
\begin{align*}
   \E{v \sim \dist}{ H_c(\s-v)-H_c(\s^*-v)}\geq \frac 1 2 (c_{\dist}(\s)-c_{\dist}(\s^*))^2
\end{align*}
\end{restatable}
The proof of the lemma is deferred to the Appendix.

\begin{proof}[Proof of \cref{thm:main}]
Recall that at every step $t\in[T]$, \cref{algo:main} runs Weitzman's algorithm as a subroutine, using an estimate $\bm \s_t$ for the optimal reservation values of the round, ${\bm \s}^*_t$.
From the robustness analysis of Weitzman's algorithm we obtain that the regret of \cref{algo:main} can be bounded as follows:
\begin{align*}
    \text{ \normalfont Regret}(\mathcal{CPB},T)
    &= \E{}{  \sum_{t\in[T]} \weitz_{\bm \dist_t}({\bm{\s}}_t; \bm c_t ) -  \weitz_{\bm \dist_t}({\bm\s}^*_t; \bm c_t)  } \\
    &\leq  \sum_{t\in[T], i\in\boxes} |c_{\dist_{t,i}}(\s_{t,i}) - c_{\dist_{t,i}}(\s^*_{t,i})| \\
    &\leq \sqrt{nT} \sqrt{ \sum_{t\in[T], i\in\boxes}  (c_{\dist_{t,i}}(\s_{t,i}) - c_{\dist_{t,i}}(\s^*_{t,i}))^2 },         &
\end{align*}
where the first inequality follows by \Cref{thm:robust} and for the last inequality we used that for any $k$-dimensional vector $z$ we have that $\|z\|_1\leq \sqrt{k}\|z\|_2$ and the fact that the above sum over $T,\boxes$ is equivalent to $\ell_1$ norm on $nT$ dimensions.
Moreover, we have that
\begin{align*}
\sum_{t\in[T], i\in\boxes}  (c_{\dist_{t,i}}(\s_{t,i}) - c_{\dist_{t,i}}(\s^*_{t,i}))^2 
& \leq  2~ \E{}{\sum_{t\in[T], i\in\boxes} \lp({H_{c_{t,i}}(\s_{t,i}-v_{t,i}) - H_{c_{t,i}}(\s_{t,i}^*-v_{t,i}})\rp)}  \\
&\leq 2\cdot n\cdot  r(T), 
\end{align*}
where for the first inequality we used Lemma~\ref{lemma:c_to_f}, and then the guarantee of the oracle. Thus, we conclude that the total expected regret is at most $2 n \sqrt{ T r(T) }$.
\end{proof}

\section{Linear Contextual Pandora's Box}\label{sec:applications}

Using the reduction we developed in Section~\ref{sec:online_regression}, we design efficient no-regret algorithms for \cpb{} in the case where the mapping from contexts to reservation values is linear. That is, we assume that $h(\bm w, \bm x) = \bm w^T \bm x$.

\subsection{Full Information Setting}\label{sec:full_info}

In this section we study the full-information version of the \cpb{} problem, where the algorithm observes the realized values of all boxes at the end of each round, irrespectively of which boxes were opened. Initially we show that there exists an online regression oracle, that achieves sublinear regret for the full information version of the Linear-Quadratic Online Regression problem. Then, in \Cref{cor:FTRL_full} we combine our reduction of \Cref{thm:main} with the online regression oracle guarantee, to conclude that \Cref{algo:main} using this oracle is no-regret for \cpb. The lemma and the theorem follow.

\begin{restatable}{lemma}{thmFTRLfullinfo}\label{thm:ftrl_full_info}
When $h(\bm w, \bm x) = \bm w^T \bm x$, $\norm{\bm{w}}{2} \leq M$ and $\norm{\bm{x}}{2}\leq 1$, there exists an oracle for Online Regression with Linear-Quadratic loss $H_c$ under \textbf{full information} that achieves regret at most $\max\{M,c\} \sqrt{2MT}$.
\end{restatable}

To show \cref{thm:ftrl_full_info}, we view Linear-Quadratic Online Regression as an instance of Online Convex Optimization and apply the \emph{Follow The Regularized Leader} (FTRL) family of algorithms to obtain the regret guarantees. 

\begin{restatable}{theorem}{corFTRLFull}\label{cor:FTRL_full}
In the full information setting, using the oracle of \Cref{thm:ftrl_full_info}, \cref{algo:main} for \cpb{} achieves a regret of
\begin{align*}
    \text{\normalfont Regret}(\mathcal{CPB}, T) \leq 3 n \sqrt{\max\{M,c_{\max}\}} M^{1/4} T^{\frac{3}{4}}, 
\end{align*}
assuming that for all times $t$ and boxes $i\in \boxes$, $\| \bm w^*_{t,i} \|_2 \le M$ and $c_{\max} = \max_{i \in \boxes} c_i$.
\end{restatable}

The process of using FTRL as an oracle is described in detail in \Cref{apn:FTRL}, alongside the proofs of \Cref{thm:ftrl_full_info} and \Cref{cor:FTRL_full}.

\subsection{Bandit Setting}\label{sec:bandit}

We move on to extend the results of the previous section to the bandit setting, and show how to obtain a no-regret algorithm for this setting by designing a regression oracle with costly feedback. In this case, the oracle $\mathcal{O}_i$ of Algorithm~\ref{algo:main} of each box $i\in \boxes$ does not necessarily receive information on the value of the box after each round. However, in each round it chooses whether to obtain the information for the box by paying the opening cost $c$.

We initially show that we can use any regression oracle given for the  full-information setting, in the costly feedback setting without losing much in terms of regret guarantees. This is formalized in the following theorem.
\begin{restatable}{lemma}{lemFullBandOracle}\label{lem:full_to_bandit_oracle}
Given an oracle that achieves  expected regret $r(T)$ for Online Regression with Linear-Quadratic loss $H_c$ under \textbf{full information}, Algorithm~\ref{algo:ftrl-bandit} is an oracle for Linear-Quadratic Online regression \textbf{with costly feedback}, that achieves regret at most $k r(T/k) + c T/k$.
\end{restatable}

Algorithm~\ref{algo:ftrl-bandit} obtains an oracle with costly feedback from a full information oracle. It achieves this by splitting the time interval $[T]$ in intervals of size $k$, and choosing a uniformly random time per interval to acquire the costly information about the input-output pair. The proof of \cref{lem:full_to_bandit_oracle} is included in Section~\ref{apn:bandit} of the Appendix.\\

\begin{algorithm}[H]
	\caption{Costly Feedback oracle from Full Information}
 	\label{algo:ftrl-bandit}
  	\KwIn{Parameter $k$, full information oracle $\mathcal{O}$}
Split the times $[T]$ into $T/k$ intervals $\mathcal{I}_1\ldots, \mathcal{I}_{T/k}$ \\
\For{Every interval $\mathcal{I}_\tau$}{
		 Pick a $t_p$ uniformly at random from $\mathcal{I}_\tau$\\
		\textcolor{blue}{\# Prediction Phase}\\
		 Call $\mathcal{O}$ to get a vector $\bm w_\tau$.
		 For each $t \in \mathcal{I}_\tau$ predict $\bm{w}_\tau$

		\mbox{}\\
		\textcolor{blue}{\# Update Phase}\\
		Obtain feedback for time $t_p \in \mathcal{I}_\tau$ and give input-output pair $(\bm x_{t_p}, y_{t_p})$ to $\mathcal{O}$.
	 }
 \end{algorithm}
\mbox{}\\
Given that we can convert an oracle for full-information to one with costly feedback using \cref{algo:ftrl-bandit}, we can now present the main theorem of this section (see \Cref{apn:bandit} for the proof):

\begin{restatable}{theorem}{thmFTRLbandit}\label{cor:main_bandit} In the bandit setting, using the oracle of \Cref{lem:full_to_bandit_oracle} together with the oracle of \Cref{thm:ftrl_full_info}, \cref{algo:main} for \cpb{} achieves a regret of
\begin{align*}
    \text{\normalfont Regret}(\mathcal{CPB}, T) \leq 2 n (2 c_{\max} M \max\{M,c_{\max} \}^2)^{1/6} T^{5/6},
\end{align*}
assuming that for all times $t$ and boxes $i\in \boxes$, $\| \bm w^*_{t,i} \|_2 \le M$ and $c_{\max} = \max_{i \in \boxes} c_i$.
\end{restatable}

\section{Conclusion and Further Directions}
We introduce and study an extension of the \pb{} model to an online contextual regime, in which the decision-maker faces a different instance of the problem at each round. We identify the minimally restrictive assumptions that need to be imposed for the problem to become tractable, both computationally and information-theoretically. In particular, we formulate a natural realizability assumption -- parallel to the one used widely in contextual bandits -- which enables leveraging contextual information to recover a sufficient statistic for the instance of each round. 
Via a reduction to Linear-Quadratic Online Regression, we are able to provide a no-regret algorithm for the problem assuming either full or bandit feedback on the realized values. We believe that the framework we develop in this work could be extended to other stochastic optimization problems beyond \pb{}. As a example, an interesting future direction would be its application to the case of revenue optimal online auctions; there Myerson's reserve prices can serve as a natural sufficient statistic (similarly to reservation prices) for obtaining optimality. In addition to applying our framework to different stochastic optimization problems, an interesting future direction would be the empirical evaluation of our methods on real data. Finally, we leave as an open question the derivation of non-trivial regret lower bounds for the considered problem. We discuss lower bounds guarantees for related settings and their implications to our problem in the Appendix.

\bibliography{reference_cam_ready}
\bibliographystyle{plainnat}

\newpage
\onecolumn

\appendix

\section{Appendix}
\subsection{Experiments} \label{sec:experiments}
We simulate the $\mathcal{CPB}$ algorithm on synthetic data, where the box values implement uniform distributions on  different intervals at each round. We simulate $n=10$ boxes with identical costs $c_i=1$ for all $i\in \boxes$, for $T=300$ rounds. The value distributions of the boxes are generated as follows:
\begin{itemize}
    \item Each box $i$ corresponds to a fixed random vector $\bm{w}^*_i\in \mathbb{R}^d$ with $d=5$, selected uniformly at the beginning of the simulation, such that $||\bm{w}^*_i||_2\leq M = 4$.
    \item At each round $t\in [T]$ and for each box $i\in\boxes$, a context $\bm{x}_{t,i}$ is uniformly drawn in $\mathbb{R}^d$ such that $||\bm{x}_{t,i}||_2\leq 1$.    
    \item The value $v_{i,t}$ of each box $i$ at time $t$ is drawn from a uniform distribution on the interval $[0,B_{t,i}]$. The right-bound $B_{t,i}$ is computed such that the reservation value $\sigma_{t,i}^*$ of the uniform distribution satisfies the realizability assumption, that is $\sigma_{t,i}^*= \bm{x}_{t,i}^T \bm{w}_i^* $.
\end{itemize} 
We implement the $\mathcal{CPB}$ algorithm for the estimation of vector $\bm{w}_{i}^*$ of each box from the observations $(\bm{x}_{t,i}^T,v_{i,t})_{t\in[T]}$ using the FTRL oracle. We compare its performance with the performance of linear regression applied to observations $(\bm{x}_{t,i}^T,v_{i,t})_{t\in[T]}$. The performance of the two methods (averaged over $20$ repetitions) is depicted in \cref{fig:plot}, along with the error bars.

\begin{figure*}[h]
    \centering
    \begin{minipage}{.44\textwidth}
        \includegraphics[width=0.9\linewidth]{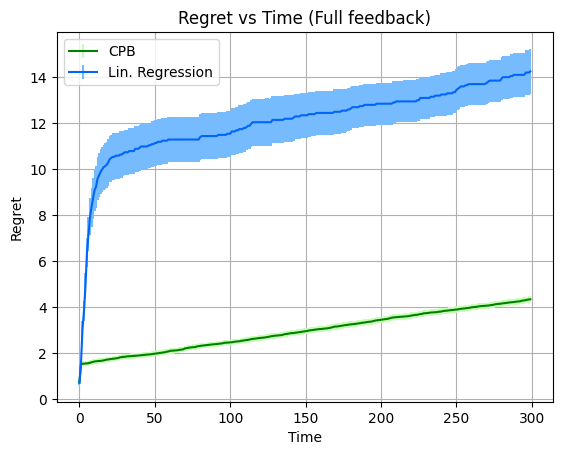}
    \end{minipage}
    \begin{minipage}{.44\textwidth}
        \includegraphics[width=0.9\linewidth]{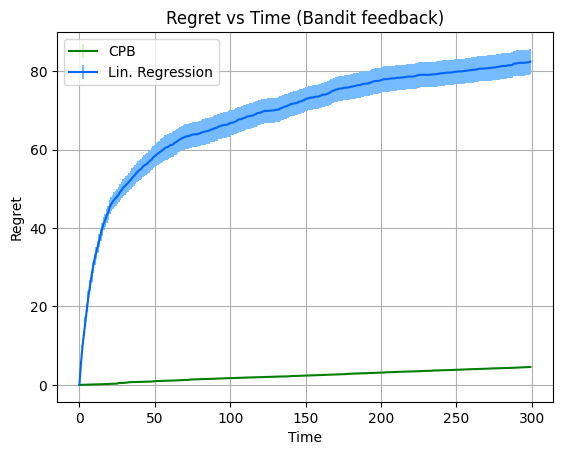}
    \end{minipage}
    \caption{The regret as a function of time in a) Full feedback setting, and b) Bandit feedback setting.} 
    \label{fig:plot}
\end{figure*}

In our first experiment, full feedback is available, i.e. the samples $(\bm{x}_{t,i}^T,v_{i,t})$ are available  to both algorithms for all $i\in\boxes,t\in[T]$. In the second experiment the algorithms have bandit feedback. That is, a sample $(\bm{x}_{t,i}^T,v_{i,t})$ is available for an algorithm only if box $i$ was opened at time $t$ by the algorithm. 
We plot the regret (defined in~\eqref{eq:regret}) of the algorithms as a function of time $t$. As expected, the regret of both algorithms is smaller under full feedback, since there is more information available at each round. In addition, the regret increases with sublinear rate as $t$ grows, which is compatible with our theoretical guarantees in both the full-feedback and the bandit setting. Finally, in both settings, using the $\mathcal{CPB}$ algorithm leads to significantly smaller regret compared to linear regression.

\subsection{Impossibility beyond Realizability}\label{apn:hardness}

We now provide an example to indicate that without the realizability assumption, our setting is not only computationally hard (even in the offline case) \citep{Kearns1992TowardEA}, but also becomes information-theoretically hard in the online case:

Consider the simple setting of two boxes, with costs either $0$ or $1$. The context provides us with some information on where is the $0$ in each round, and our objective it to select $0$ the maximum amount of times. 
Clearly, in the offline case the above setting coincides with that of agnostic learning, which is known to be computationally hard, even for linear functions (e.g., linear classification).

In the online setting, where an adversary can choose the context to give us at each round, the problem is information-theoretically hard. Similarly, assume that a context is a real number in $[0,1]$ and that there exists a threshold $x$, which decides which box gives $0$ and which $1$. The adversary can always give us a threshold in the uncertainty region, and force us to make a mistake at every round, thus accumulating linear cost over time, while the cost accumulated by the optimal algorithm (which knows the context-value relation) is 0.





\subsection{Proofs from Section~\ref{sec:robust}}\label{apn:robust}

\propCost*
\begin{proof}

Assume we are using Weitzman's Algorithm with reservation values 
$\bm \s$ in an instance with optimal reservation values $\bm\s^*$ and costs $\bm c$. Let $\mathcal{P}$ be the set of boxes opened during the algorithm's run.
Then, the \pb{} cost incurred at the end of the run can be bounded as follows:
\begin{align*}
\weitz_{\bm \dist}(\bm{\s}; \bm c) &= \E{\bm v\sim \bm \dist}{\min_{i\in {\mathcal{P}}} v_i + \sum_{i\in \mathcal{P}} c_i} &\\
& = \E{\bm v\sim \bm \dist}{\min_{i\in {\mathcal{P}}} v_i + \sum_{i\in \mathcal{P}} c_i + c_{\dist_i}(\s_i) - c_{\dist_i}(\s_i)} &\\
& =  \weitz_{\bm \dist}(\bm\s ; \bm c_{\bm \dist}(\bm \s)) + \E{\bm v\sim \bm \dist}{\sum_{i\in {\mathcal{P}}}(c_i - c_{\dist_i}(\s_i))}&\\
&\leq  \weitz_{\bm \dist}({\bm\s}^* ; \bm c) +\\ &\qquad \E{\bm v\sim \bm \dist}{\sum_{i\in B} \relu(c_{\dist_i}(\s_i) - c_i) +\sum_{i\in {\mathcal{P}}}(c_i - c_{\dist_i}(\s_i))}  & \\
& \leq \weitz_{\bm \dist}({\bm\s}^* ; \bm c) + \E{\bm v\sim \bm \dist}{\sum_{i\in \boxes} \lp|c_i-c_{\dist_i}(\s_i)\rp| },&
\end{align*}
where in the first inequality we used Lemma~\ref{lem:opt_hat} and in the last inequality that every box from each of the sets $\boxes$ and $\mathcal{P}$ can contribute at most once in the sum, because of the ReLU function.
\end{proof}

\lemOptHat*
\begin{proof}
Consider the optimal strategy $\weitz_{\bdist}(\bm{\s}^*; \bm{c})$ for the instance with reservation values $\bm{\s}^*$, and assume we use the same strategy in the instance where the `actual reservation values are $\hat{\bm{\s}}$. That means our algorithm $\weitz_{\bm{\dist}}(\bm{\s}^*;\bm{c}_{\bm \dist}(\bhs))$ orders the boxes according to the reservation values $\bm{\s}^*$ and stops when $\min_i v_i \leq \s^*_{i+1}$. Denote by $\mathcal{P}$ the set of boxes probed (opened) by the algorithm $\weitz_{\bdist}(\bm{\s}^*;\bm{c}_{\bdist}(\bhs))$, then the cost is
\begin{align*}
\weitz_{\bdist}(\bhs;\bm{c}_{\bdist}(\bhs)) & \leq\weitz_{\bdist}(\bm{\s}^*;\bm{c}_{\bm \dist}(\bhs)) \\
    & = \E{\bm v\sim \bm D}{\min_{i\in \mathcal{P}} v_i + \sum_{i\in \mathcal{P}}c_{\dist_i}(\hs_i)} \\
    & =\E{\bm v\sim \bm D}{\min_{i\in \mathcal{P}} v_i} + \E{\bm{v} \sim \bm{D}}{\sum_{i\in \mathcal{P}}c_{\dist_i}(\hs_i) - c_{\dist_i}(\s^*_i) + c_{\dist_i}(\s^*_i)} \\
    & = \weitz_{\bdist}(\bm{\s}^*; \bm{c})+ \sum_{i\in \mathcal{P}} (c_{\dist_i}(\hs_i) - c_{\dist_i}(\s^*_i))\\
    & \leq \weitz_{\bdist}(\bm{\s}^*; \bm{c}) + \sum_{i\in \mathcal{P}} \relu(c_{\dist_i}(\hs_i) - c_{\dist_i}(\s^*_i)),
\end{align*}
where the first inequality follows by the definition of the optimal, and the first equality is the actual cost of our algorithm, since in our instance we pay $c_{\dist_i}(\hs_i)$ for each box. 
\end{proof}

\subsection{Proofs from \cref{sec:concluding_full_info_th}}\label{apn:main_thm}

\lemmaCtoF*

\begin{proof}

Recall from Definition~\ref{def:online_regression} that
\begin{align*}
H_c(\s - v) =  \frac{1}{2} \relu(\s - v)^2 - c (\s - v) = \left\{\begin{array}{lr}
        \frac{1}{2} (\s - v)^2 - c (\s - v) & \text{if } \s \geq v \\
        - c (\s - v), & \text{if }\s < v 
        \end{array}\right\}
\end{align*}
and 
$$ H'_c(\s - v) = \relu(\s-v) - c.$$ 
We also use that for $\s=\s^*$ using Weitzman's theorem for optimal reservation values (Theorem~\ref{thm:weitzman}) we have 
\[ 
\E{v}{H'_c(\s^*-v)}=0.
\]

We need to compare $H_c(\s - v)-H_c(\s^* - v)$ to  $|c_D(\s)-c_D(\s^*)|$. 
Using that $\int_{a}^b f'(x)\,dx = f(b) - f(a)$ and changing the order between expectation and integration, we obtain
\begin{align}\label{eq:int_1}
    \E{v \sim D}{H_c(\s - v)-H_c(\s^* - v)} &=  \E{v \sim D}{\int_{\s^*}^{\s}H'_c(\s' - v)\,d\s'} \nonumber\\
    &= \int_{\s^*}^{\s} \E{v\sim D}{H'_c(\s' - v)}\,d\s' \nonumber\\
    &= \int_{\s^*}^{\s} \E{v\sim D}{H'_c(\s' - v)} - \E{v\sim D}{H'_c(\s^* - v)}\,d\s' \nonumber\\
    &= \int_{\s^*}^{\s} \E{v\sim D}{\relu(\s'-v)} - \E{v\sim D}{\relu(\s^*-v)}\,d\s' \nonumber\\
    &= \int_{\s^*}^{\s} (c_D(\s') - c_D(\s^*)) \,d\s',
\end{align}
where the last inequality above is by definition of $c_D(\cdot)$. We now distinguish between the following cases for the reservation value $\s$ compared to the optimal reservation value $\s^*$.

\paragraph{Case $\s\geq \s^*$:} 
observe that $c_D$ is 1-Lipschitz, therefore we have that $|c_D(\s) - c_D(\s')| < |\s - \s'| $.
Moreover, when $\s' \ge \s^*$ we have $c_D(\s' ) \ge c_D(\s^*)$. Thus \cref{eq:int_1} becomes:
\begin{align*}
    \E{v \sim D}{H_c(\s - v)-H_c(\s^* - v)} &= \int_{\s^*}^{\s} (c_D(\s') - c_D(\s^*)) \,d\s' \\ 
    &\ge \int_{\s - (c_D(\s) - c_D(\s^*))}^{\s} (c_D(\s') - c_D(\s^*)) \,d\s' \\
    &\ge \int_{\s - (c_D(\s) - c_D(\s^*)) }^{\s} (c_D(\s) - c_D(\s^*) - (\s - \s') ) \,d\s' \\
    &= \frac 1 2 (c_D(\s) - c_D(\s^*))^2,
\end{align*}
where for the first inequality we used that $ \s - \s^* \geq c_D(\s) - c_D(\s^*) $ and for the second that $\s - \s' \geq c_D(\s) - c_D(\s') $.

\paragraph{Case $\s<\s^*$:} 
as before, \cref{eq:int_1} can be written as
\begin{align*}
    \E{y \sim D}{H_c(\s - v)-H_c(\s^* - v)} &= \int_{\s^*}^{\s} (c_D(\s') - c_D(\s^*)) \,d\s' \\ 
    &= - \int_{\s}^{\s^*} (c_D(\s') - c_D(\s^*)) \,d\s' \\ 
    &\ge -\int_{\s}^{\s - (c_D(\s) - c_D(\s^*))} (c_D(\s') - c_D(\s^*)) \,d\s' \\
    &\ge \int_{\s}^{\s - (c_D(\s) - c_D(\s^*)) } (-c_D(\s) + c_D(\s^*) - (\s' - \s) ) \,d\s' \\
    &= \frac 1 2 (c_D(\s) - c_D(\s^*))^2,
\end{align*}
where for the first inequality we used that $ \s^* - \s \geq c_D(\s^*) - c_D(\s) $ and the fact that $c_D(\s') - c_D(\s^*)\leq 0$. For the second we used that $- c_D(\s') \geq -c_D(\s) -(\s' -\s)  $.

\end{proof}

\subsection{Regression via FTRL and Proofs of \cref{sec:full_info}}\label{apn:FTRL}

Recall our regression function defined in equation~\eqref{eq:convex_proxy}. We define $f_t: \mathbb{R}^d \rightarrow \mathbb{R}$, as the loss function at time $t\in [T]$ as 
\[
 f_t(\bm{w}) = H_c(\bm{w}^T \bm{x}_{t}-y_t) = \frac{1}{2} (\bm{w}^T \bm{x}_{t} - y_t)_+^2 - c (\bm{w}^T \bm{x}_{t} - y_t). 
\]

As these functions are convex, we can treat the problem as an online convex optimization. Specifically, the problem we solve is the following.
\begin{enumerate}
    \item At every round $t$, we pick a vector $\bm{w}_{t}\in \mathbb{R}^d$.
    \item An adversary picks a convex function $f_t: \mathbb{R}^d \rightarrow \mathbb{R}$ induced by the input-output pair $(\bm x_t, y_t)$ and we incur loss $f_{t}(\sigma_t)$.
    \item At the end of the round, we observe the function $f_{t}$.
\end{enumerate}

In order to solve this problem, we use a family of algorithms called Follow The Regularized Leader (FTRL). In these algorithms, at every step $t$ we pick the solution $\bm{w}_t$ that would have performed best so far while adding a regularization term $U(\bm {w}_t)$ for stability reasons. That is, we choose $$\bm w_t = \arg \min_{\bm w : \|w\|_2 \le M} \sum_{\tau=1}^{t-1} f_t(\bm w) + U(\bm w).$$ The guarantees for these algorithms are presented in the following lemma. 

\begin{lemma}[Theorem 2.11 from \cite{Shwa2012}]
\label{lem:FTRL}
	Let $f_1,\ldots ,f_T$ be a sequence of convex functions such that each $f_t$ is $L$-Lipschitz with respect to some norm.  Assume that FTRL is run on the sequence with a regularization function $U$ which is $\eta$-strongly-convex with respect to the same norm. Then, for all $u\in C$ we have that $\text{Regret}(\text{FTRL}, T) \cdot T \leq U_{\max} - U_{\min} + TL^2 \eta$
\end{lemma}

Observe that in order to achieve the guarantees of \cref{lem:FTRL}, we need the functions to have some convexity and Lipschitzness properties, for which we show the following lemma. 

\begin{lemma}\label{lem:convexity}
   The function $f_t(\bm w) = H_c(\bm w^T \bm x_t - y_t)$ is convex and $\max\{M,c\}$-Lipschitz when $\|x_t\|_2 \le 1$ and $\|\bm w\|_2 \le M$.
\end{lemma}
\begin{proof}[Proof of Lemma~\ref{lem:convexity}]
We first show convexity and Lipschitzness for the function $H_c(z) = \frac 1 2 \relu(z)^2 - cz$  for  $z \in \reals$.

Consider the derivative $H'_c(z) =  \relu(z) - c$ and the second derivative $H''_c(z) =  \textbf{1}(z \ge 0)$ and notice that the second derivative is always non-negative which implies convexity.

To bound the Lipschitzness of $H_c$, we consider the maximum absolute value of the derivative which is at most $\max \{c, \relu(z)\}$.

We now turn our attention to $f_t$ and notice that $f_t$ is convex as a composition of a convex function with a linear function. To show Lipschitzness we must bound the norm of the gradient of $f_t$ which is:
$$\| \nabla f_t(\bm w) \|_2 = | H'_c(\bm w^T \bm x_t - y_t) | \| \bm x \|_2 \le \max \{c, M\},$$
where the last inequality follows as the maximum value of $\bm w^T \bm x_t - y_t$ is at most $M$.
\end{proof}
\thmFTRLfullinfo*
\begin{proof}
We use the following regularizer $U(\bm{w}) = \frac{1}{2\eta} \norm{\bm{w}}{2}^2$, and observe that this is $\eta$-strongly convex with $U_{\min} = 0$ and $U_{\max} = M/(2\eta)$.
From \cref{lem:convexity} we have that the loss function is convex and $\max \{c, M\}$-Lipschitz, therefore using \cref{lem:FTRL} with $\eta = \frac {\sqrt{M}} {\max \{c, M\} \sqrt{2T}}$ we get that $\text{Regret} (FTRL,T) \leq \frac M {2 \eta} + \eta (\max \{c, M\})^2 T = \max \{c, M\} \sqrt{ 2 M T }$.
\end{proof}

\corFTRLFull*
\begin{proof}
The regret of \cpb~ can be upper bounded as follows:
    \begin{align*}
    \text{\normalfont Regret}(\mathcal{CPB}, T) &\leq 2 n \sqrt{ Tr(T) } &\text{by \Cref{thm:main}}\\
    &\leq 3 n \sqrt{\max\{M,c\}} M^{1/4} T^{\frac{3}{4}}.  &\text{by Lemma \ref{thm:ftrl_full_info}}
    \end{align*}
\end{proof}

\subsection{Proofs from Section~\ref{sec:bandit}}\label{apn:bandit}

In this section we show how to obtain the guarantees of FTRL in the case of costly feedback proving Lemma~\ref{lem:full_to_bandit_oracle}.

While the proof is standalone, the analysis uses ideas from \cite{GergTzam2022} (in particular Lemma~4.2 and Algorithm~2). We give a simplified presentation for the case of online regression with improved bounds.

\lemFullBandOracle* 
 
\begin{proof}
 Recall that in the online regression problem, we obtain loss $f_t$ at any time step where $f_t(\bm{w}) = H_{c}(\bm{w}^T \bm{x}_t - y_t) =  \frac{1}{2} \relu(\bm{w}^T \bm{x}_t - y_t)^2 - c (\bm{w}^T \bm{x}_t - y_t)$.
 
 To analyze the regret for the costly feedback setting, we consider the regret of two related settings for a full-information online learner but with smaller number of time steps $T/k$.
\begin{enumerate}
    \item \textbf{Average costs setting}: the learner observes at each round $\tau$ a single function $$\overline{f}_{\tau} = \frac 1 k \sum_{t \in \mathcal{I}_\tau} f_{t},$$ which is the average of the $k$ functions in the corresponding interval  $\mathcal{I}_\tau$.
    \item \textbf{Random costs setting}: the learner observes at each round $\tau$ a single function $$f^r_\tau = f_{t_p} \quad \text{with} \quad t_p \sim \text{Uniform}(\mathcal{I}_\tau)$$ sampled uniformly among the $k$ functions
    $f_t$ for $t \in \mathcal{I}_\tau$.
 \end{enumerate}
 
 The guarantee of the full information oracle implies that for the random costs setting we obtain regret $r(T/k)$. That is, the oracle chooses a sequence $\bm{w}_1,\ldots,\bm{w}_{T/k}$ such that
$$ \sum_{\tau =1}^{T/k}
	f^r_\tau (\bm{w}_\tau) \le \min_{\bm w} \sum_{\tau =1}^{T/k}
	f^r_\tau (\bm{w}) + r(T/k).
$$

 Denote by $\bm w^* = \text{argmin}_{\bm{w}\in \mathbb{R}^d} \sum_{\tau =1}^{T/k}
	\overline{f}_\tau (\bm{w})$ be the minimizer of the $\overline{f}_t$ over the $T/k$ rounds. Note that this is also the minimizer of 
    $\sum_{t =1}^{T}
	{f}_t (\bm{w})$. From the above regret guarantee, we get that
	
$$ \sum_{\tau =1}^{T/k}
	( f^r_\tau (\bm{w}_\tau)  - f^r_\tau (\bm{w}^*)  ) \le  r(T/k).
$$ 
Since $\E{}{f^r_\tau(\bm{z}_\tau) } = \overline{f}_\tau(\bm{z}_\tau)$ when the expectation is over the choice of $t_p \sim \text{Uniform}(\mathcal{I}_\tau)$. Taking expectation in the above, we get that
$$ \sum_{\tau =1}^{T/k}
	( \overline{f}_\tau (\bm{w}_\tau)  - \overline{f}_\tau (\bm{w}^*)  ) \le  r(T/k),
$$ 
which implies that the regret in the average costs setting is also $r(T/k)$.

We now obtain the final result by noticing that the regret of the algorithm is at most $k$ times the regret of the average costs setting and incurs an additional overhead of $\frac T k c$ for the information acquisition cost in the $T/k$ rounds where the input-output pairs are queried. Thus, overall the regret is bounded by $k r(T/k) + c T / k$.
\end{proof}

\thmFTRLbandit*
\begin{proof}
Initially observe that by combining \cref{thm:ftrl_full_info} and \cref{lem:full_to_bandit_oracle} we get that there exists an oracle for online regression with loss $H_c$ under the costly feedback setting that guarantees 
\begin{align*}
\text{Regret}(\mathcal{O}_{\text{costly}}, T) &\leq L \sqrt{2kMT} + c T /k
\end{align*}
with $L =\max\{M,c\}$.
Setting $k= \left( \frac {T c^2}  {2 M L^2} \right)^{1/3}$, we obtain regret $(2 c M L^2)^{1/3} T^{2/3}$. Further combining this with Theorem~\ref{thm:main} that connects the regret guarantees of regression with the regret for our \cpb{} algorithm, we obtain regret $2 n (2 c_{\max} M \max\{M,c_{\max} \}^2)^{1/6} T^{5/6}$.
\end{proof}

\subsection{Discussion on the Lower Bounds}
 In this work our goal was to formulate the online contextual extension of the Pandora’s Box problem and design no-regret algorithms for this problem, and therefore we left the lower bounds as a future work direction. We are however including here a brief discussion on lower bounds implied by previous work.

In \citet{GatmKessSingWang2022} the authors study a special case of our setting where the value distributions and contexts are fixed at every round and for each alternative. The results of \cite{GatmKessSingWang2022} imply a lower bound for this very special case of our problem, which however does not correspond to an adversarial setting, as our general formulation does. Notice, also, that a $\sqrt{T}$ lower bound for our problem can be directly obtained by the fact that our setting generalizes the stochastic multi-armed bandit problem. Observe that if all costs are chosen to be identical and large enough, then both the player and the optimal solution must select exactly one alternative per round (and their inspection costs cancel out in the regret). Interestingly, in that case Weitzman’s algorithm indeed selects the alternative of the smallest mean reward.

In another related setting, multi-armed bandits with paid observations, \citet{SeldBartCramAbba2014} show a $T^{2/3}$
lower bound on the regret in the adversarial case. Although their setting is different (e.g. in our setting multiple actions are allowed at each round and there is contextual information involved) we believe that it has similarities to ours in terms of costly options and information acquisition. Therefore, we believe that tighter lower bounds could hold for our general adversarial problem.

\subsection{A primer on Pandora's Box}
We include here a more detailed discussion on Pandora's Box problem for anyone unfamiliar with the setting. 

This problem was first formulated by Weitzman~\cite{Weit1979} in the following form; there are $n$ boxes, each box $i$ has a deterministic opening cost $c_i\geq 0$ and a value that follows a known distribution $\mathcal{D}_i$. The player can observe the value instantiated by the distribution of a box after paying the opening cost. The goal is to select boxes to open, and a value to keep so as to maximize 

\[
\E{\mathcal{P}, \mathcal{D}}{\max_{i\in \mathcal{P}} v_i + \sum_{i\in \mathcal{P}}c_i}
\]

In the original problem, the distributions $\mathcal{D}_i$ are independent, and the optimal algorithm is given by Algorithm~\ref{algo:optimal_maximization}. In the algorithm Weitzman used the reservation value for each box $i$ which is a number calculated as the solution to the equation $\E{\mathcal{D_i}}{(v-\s_i)^+} = c_i$, and as shown in \cite{Weit1979}, this is enough to produce an optimal solution through Algorithm~\ref{algo:optimal_maximization}. In our work we tackle the minimization version of the problem (similarly to \cite{ChawGergTengTzamZhan2020, GergTzam2022, ChawGergMcmaTzam2021}), described in more detail in Section~\ref{sec:prelims}. For the independent case the same results hold regardless of the version and in Algorithm~\ref{algo:optimal_maximization} we highlighted the differences with Algorithm~\ref{algo:weitz_prelim} to show the changes required for the maximization.

\begin{algorithm}
   \caption{Weitzman's algorithm.}\label{algo:optimal_maximization}
 \textbf{Input:} $n$ boxes, costs $c_i$, reservation values $\s_i$ for $i\in \boxes$\\
\begin{algorithmic}[1]
	\STATE $\pi \leftarrow $ sort boxes by \textbf{decreasing} $\s_i$\\
	\STATE $v_{\textbf{max}} \leftarrow 0$\\
	\FOR{every box $\pi_i$ } 
			\STATE Pay $c_i$ and open box $\pi_i$ and observe $v_i$\\
			\STATE $v_{\textbf{max}} \leftarrow \textbf{max}(v_{\textbf{max}}, v_i)$\\
			\IF{$v_{\textbf{max}} < {\s}_{\pi_{i+1}}$}
				\STATE Stop and collect $v_{\textbf{max}}$
			\ENDIF
\ENDFOR
\end{algorithmic}
\end{algorithm}

For a summary of the recent work on Pandora's Box see the survey~\cite{BeyhCai2023}.


\end{document}